\documentclass{article} 
\usepackage{iclr2027_conference,times}

\usepackage{amsmath,amsfonts,bm}

\def\eqref#1{equation~\ref{#1}}

\def\1{\bm{1}}

\DeclareMathAlphabet{\mathsfit}{\encodingdefault}{\sfdefault}{m}{sl}
\SetMathAlphabet{\mathsfit}{bold}{\encodingdefault}{\sfdefault}{bx}{n}

\usepackage{hyperref}
\usepackage{url}

\usepackage{amsfonts}
\usepackage{amsmath}
\usepackage{amssymb}

\usepackage{amsthm}
\newtheorem{proposition}{Proposition}
\newtheorem{corollary}{Corollary}

\usepackage{booktabs}
\usepackage{multirow}
\usepackage{graphicx}

\usepackage{pifont}   
\usepackage{xspace}   

\usepackage{makecell}

\usepackage[table,dvipsnames]{xcolor} 
\definecolor{red4}{HTML}{febf92}

\newcommand{\icoyes}{\textcolor{ForestGreen}{\ding{51}}\xspace} 
\newcommand{\icono}{\textcolor{Red}{\ding{55}}\xspace}          
\definecolor{myLightPurple}{HTML}{E9E2F2} 

\newcommand{\pos}[1]{\textcolor{green!50!black}{\scriptsize#1}}
\newcommand{\nega}[1]{\textcolor{red!70!black}{\scriptsize#1}}

\newcommand{\score}[2]{%
  \makebox[6.2em][c]{%
    #1%
    \makebox[0pt][l]{\hspace{0.2em}#2}%
  }%
}

\usepackage{algorithm}
\usepackage{algorithmic}

\usepackage{titletoc}

\usepackage{wrapfig}

\title{Flux-OPD: On-Policy Distillation with \\Evolving Contexts
\vspace{-0.5em}}

\author{\small\bfseries Yuran Wang\textsuperscript{1,2},
    Zekun Wang\textsuperscript{2},
    Bohan Zeng\textsuperscript{1},
    Ruixu Zhang\textsuperscript{3},
    Wenxuan Liu\textsuperscript{1},
    Liu Yang\textsuperscript{4},
    Yifan Dai\textsuperscript{4},\\
    \small\bfseries
    Yang Shi\textsuperscript{1},
    Bozhou Li\textsuperscript{1},
    Chengzhuo Tong\textsuperscript{1},
    Daili Hua\textsuperscript{1},
    Yuanxing Zhang\textsuperscript{2},
    Wentao Zhang\textsuperscript{1,5,6}
    \thanks{Corresponding Author: wentao.zhang@pku.edu.cn}\\
    \small\normalfont
    \textsuperscript{1}Peking University,
    \textsuperscript{2}Kling Team,
    \textsuperscript{3}Tsinghua University,
    \textsuperscript{4}Shanghai Jiao Tong University,\\
    \small\normalfont
    \textsuperscript{5}Zhongguancun Academy,
    \textsuperscript{6}Beijing Key Laboratory of Data Intelligence
    and Security (Peking University)
}

\iclrfinalcopy 
\begin{document}

\maketitle
\lhead{Preprint}
\vspace{-1.5em}
\begin{abstract}
Large language model training in open-ended domains lacks verifiable rewards, making task preferences difficult to formalize as effective supervision. Contexts derived from task feedback can provide richer supervision, yet fixed contexts cannot adapt as the student evolves, motivating contexts that evolve with student performance. However, directly using evolving contexts as in-training supervision results in unstable distillation targets and conflicting distributions, requiring mechanisms to stabilize the target and downweight conflicts. In this paper, we characterize context distillation through a reverse-KL decomposition, showing that the student is distilled toward the geometric mean of context-conditioned teachers, while the objective contains a conflict term measuring disagreement among these teachers. Motivated by this characterization, we propose Flux-OPD, an OPD paradigm that uses evolving contexts as in-training supervision to incorporate contextual signals derived from task feedback. Flux-OPD treats differences between context-conditioned and context-free teachers as contextual difference signals, injects them as contextual corrections into the context-free teacher, and weights the correction strength using the conflict term. Experiments show that Flux-OPD achieves the best results among OPD paradigms, highlighting the potential of combining teacher supervision with evolving contexts as in-training supervision.
\end{abstract}

\vspace{-1em}
\section{Introduction}
Large language models achieve remarkable progress through Reinforcement Learning with Verifiable Rewards in verifiable domains~\citep{shao2024deepseekmath}.
However, open-ended domains often lack verifiable rewards, making it difficult to formalize supervision signals that effectively capture task preferences.
Rubrics-as-Rewards (RaR)~\citep{gunjal2025rubrics,xiao2026structuring} uses rubrics as supervision but still provides sparse feedback via a scalar score, limiting its effectiveness in capturing complex preference priors~\citep{ye2026llm,lu2025onpolicydistillation,schulman2025lora}.

An alternative is to use contexts as supervision signals.
Context Distillation and On-Policy Context Distillation (OPCD)~\citep{askell2021general,ye2026policy} distill the student toward a context-conditioned teacher, enabling it not only to benefit from dense teacher supervision, as in vanilla On-Policy Distillation (OPD)~\citep{agarwal2024policy}, but also to directly learn behaviors that the teacher exhibits after observing the contexts.
Thus, explicit contexts can provide additional task-related supervision that is difficult to formalize in open-ended domains.
However, we find that directly distilling toward the context-conditioned teacher (OPCD) underperforms distilling toward the context-free teacher (OPD), as shown in Fig.~\ref{fig:motivation}.
Since these contexts are extracted before training, they remain fixed as the student evolves.
This motivates us to \textbf{explore contexts that evolve with student performance}.

\begin{figure}[t]
    \centering
    \includegraphics[width=1\linewidth]{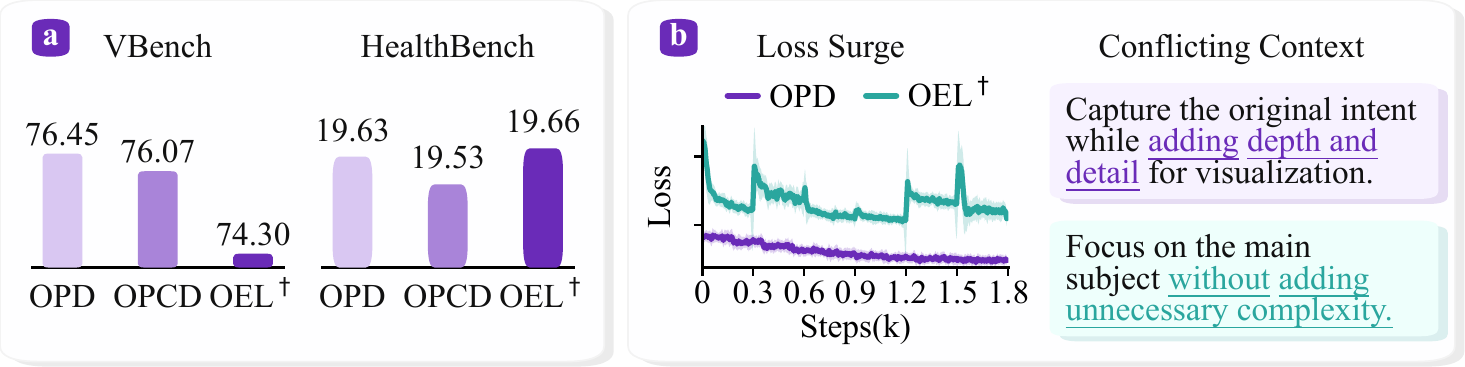}
    \vspace{-1.5em}
    \caption{
    Motivations.
    (a) OPCD underperforms OPD, while OEL$^\dagger$ does not yield consistent gains.
    (b) Context updates are accompanied by loss surges, while conflicts among contexts can produce conflicting distributions.
    $^\dagger$ denotes our single-run adaptation of OEL.
    }
    \label{fig:motivation}
    \vspace{-0.5em}
\end{figure}

The key challenge is to incorporate evolving contexts as supervision signals throughout training without disrupting optimization.
Online Experiential Learning (OEL)~\citep{ye2026online} alternates between deployment and OPCD training, accumulating experience items from extensive student trajectories as contexts.
To provide in-training supervision, we extend this paradigm to a single training run, using only a small set of trajectories at each update to provide timely feedback on the current student.
However, as shown in Fig.~\ref{fig:motivation}, evolving contexts do not yield consistent gains.
Each update can abruptly shift the target distribution, accompanied by loss surges and less stable training.
Moreover, conflicting contexts can produce conflicting target distributions.
These observations motivate us to \textbf{stabilize the distillation target and downweight conflicting distributions}.

To better leverage contexts, we analyze their effect on the reverse KL objective in the OPCD paradigm.
We show that this objective can be decomposed into \textit{a distillation term} and \textit{a conflict term}.
The former computes the KL divergence between the student and the geometric mean of context-conditioned teacher distributions, while the latter measures conflicts among these teachers.
Conditioned on fixed sampled histories, the decomposition shows that the conflict term contributes no direct gradient, leaving the distillation term as the directly optimized component of the objective.

Motivated by this characterization, we propose \textbf{Flux-OPD}, an OPD paradigm that uses evolving contexts as in-training supervision to incorporate contextual signals derived from task feedback.
It is an iterative training process that divides a single training run into multiple iterations, each comprising context extraction and context distillation.
The teacher extracts contexts from student trajectories during context extraction and conditions on these contexts during context distillation.

We design two distillation strategies: \textbf{contextual correction} and \textbf{contextual weighting}.
To stabilize the distillation target, the contextual correction strategy anchors it to the relatively stable context-free teacher and injects contextual difference signals that capture task preferences from evolving contexts, rather than directly treating context-conditioned teachers as distillation targets.
Contextual difference signals are defined as the differences between the context-conditioned and context-free teachers.
To downweight conflicting distributions, the contextual weighting strategy adjusts the correction strength based on the conflict term, strengthening corrections when context-conditioned distributions are consistent and weakening them when they conflict.

We evaluate our method on two open-ended tasks with real-world applicability: prompt optimization for video generation and medical question answering, both of which involve complex preferences and rely on rich contextual information.
Results show that Flux-OPD outperforms existing OPD paradigms on both tasks across different student-teacher settings.
Our contributions are threefold:
\begin{itemize}    
    \item We explore evolving contexts as in-training supervision for on-policy distillation in open-ended domains. We identify two key challenges in directly incorporating such supervision: unstable distillation targets and conflicting context-conditioned distributions.

    \item We characterize context distillation through a reverse-KL decomposition into a geometric-mean target and a conflict term. This motivates Flux-OPD with contextual correction and weighting, combining stable teacher and evolving contextual supervision.
    
    \item Experiments on open-ended tasks show that Flux-OPD achieves the best reported performance among the evaluated OPD paradigms, highlighting the potential of combining teacher supervision with evolving contexts.
\end{itemize}

\section{Related Work}

\subsection{Open-Ended Reward Modeling}
Large language model training in open-ended domains lacks verifiable rewards, making task preferences difficult to formalize as effective supervision.
Reinforcement Learning with Verifiable Rewards (RLVR) compresses these fine-grained preferences into scalar rewards.
Rubrics-as-Rewards (RaR)~\citep{gunjal2025rubrics,xiao2026structuring} uses rubrics as structured supervision but still provides sparse feedback.
On-Policy Distillation (OPD) trains on trajectories sampled from the current student policy and provides dense teacher supervision~\citep{agarwal2024policy,zhao2026self,self-distillation,fu2026revisiting,gu2024minillm}, but does not directly incorporate
contextual signals derived from task feedback~\citep{gudibande2023false,yang2026self,yang2026learning}.
Flux-OPD complements teacher supervision with evolving contexts derived from task feedback as in-training supervision.

\subsection{Context Distillation}
Context Distillation~\citep{askell2021general,snell2022learningdistillingcontext} distills a student from a context-conditioned teacher, enabling it to learn behaviors that the teacher exhibits after observing the contexts.
On-Policy Context Distillation (OPCD)~\citep{ye2026policy} further samples from the current student, reducing the mismatch between offline data and the student's generation distribution.
Although contexts can provide additional task-related supervision in
open-ended domains, contexts fixed before training cannot adapt to the evolving student.
OEL~\citep{ye2026online} alternates between deployment and OPCD training, updating contexts between runs.
Directly extending this paradigm to a single continuous run results in an unstable distillation target and conflicting distributions.
Flux-OPD instead uses evolving contexts as in-training supervision by distilling contextual difference signals into a context-free anchor rather than fully imitating context-conditioned teachers.

\section{Preliminaries}
\label{sec:preliminary}

\subsection{Paradigms}

In vanilla On-Policy Distillation (OPD), given a user prompt $x$, the student $\pi_\theta$ autoregressively generates a response $y=(y_1,\ldots,y_T)$.
Let $h_t=(x,y_{<t})$ denote the decoding history and $\mathcal{V}$ the vocabulary. At step $t$, the student and teacher distributions are defined as
\begin{equation}
\begin{aligned}
p_\theta(v \mid h_t) &= \pi_\theta(y_t = v \mid h_t),\\
q_0(v \mid h_t) &= \pi_{\mathrm{T}}(y_t = v \mid h_t),
\quad v \in \mathcal{V}.
\end{aligned}
\end{equation}

In On-Policy Context Distillation (OPCD), the teacher additionally receives a privileged context $c\sim\mathcal{C}$, while the student observes only $h_t$.
The context-conditioned teacher distribution is
\begin{equation}
q_c(v\mid h_t)=\pi_{\mathrm{T}}(y_t=v\mid h_t,c).
\end{equation}
Training minimizes the reverse KL under on-policy histories:
\begin{equation}
\mathcal{L}(\theta)
=
\mathbb{E}_{h_t}\mathbb{E}_{c\sim\mathcal{C}}
\left[
D_{\mathrm{KL}}
\left(
p_{\theta}(\cdot\mid h_t)
\|
q_c(\cdot\mid h_t)
\right)
\right].
\end{equation}
For simplicity, fixing $h_t$ and omitting the condition gives
\begin{equation}
\mathcal{L}(\theta)
=
\mathbb{E}_{c\sim\mathcal{C}}
\left[
D_{\mathrm{KL}}(p_\theta\|q_c)
\right].
\label{eq:reverse_kl}
\end{equation}

\subsection{KL Decomposition}

For a fixed decoding history $h_t$ at step $t$, let
$p_\theta(v)=p_\theta(v\mid h_t)$ and $q_c(v)=q_c(v\mid h_t)$.
We omit $h_t$ for readability.

\paragraph{Forward KL}
Prior work~\citep{yang2026self} decomposes the forward KL objective into a distillation term toward a marginal teacher and a mutual information term:
\begin{equation}
\mathbb{E}_{c}
\left[
D_{\mathrm{KL}}(q_c\|p_\theta)
\right]
=
\underbrace{
D_{\mathrm{KL}}(q_{\mathrm{mar}}\|p_\theta)
}_{\text{distillation term}}
+
\mathbb{E}_{c}
\left[
D_{\mathrm{KL}}(q_c\|q_{\mathrm{mar}})
\right],
\label{eq:forward_kl_decom_conclusion}
\end{equation}
where the marginal teacher is defined as
\begin{equation}
q_{\mathrm{mar}}(v)
=
\mathbb{E}_{c}[q_c(v)].
\end{equation}
Thus, under forward KL, the student is distilled toward the \textbf{arithmetic mean} of context-conditioned teachers.

\paragraph{Reverse KL}
\label{sec:reverse-kl-decomposition}

We characterize context distillation through the objective in Eq.~(\ref{eq:reverse_kl}).

\begin{proposition}[Reverse KL Decomposition]
\label{prop:reverse-kl-decomposition}
For a fixed decoding history, let $p_\theta$ and $\{q_c\}$ be distributions over the same vocabulary $\mathcal{V}$, where $q_c(v)>0$ for every context $c$ and token $v\in\mathcal{V}$.
Let $\mathbb{E}_c$ denote expectation with respect to a fixed context distribution.
Define the normalized geometric mean of context-conditioned teachers and its normalization constant as
\begin{equation}
\begin{aligned}
q_{\mathrm{geo}}(v)
=
\frac{
\exp\!\left(\mathbb{E}_{c}[\log q_c(v)]\right)
}{Z},
Z
=
\sum_{v \in \mathcal{V}}
\exp\!\left(\mathbb{E}_{c}[\log q_c(v)]\right).
\end{aligned}
\label{eq:geometric-mean}
\end{equation}
Then, the reverse KL objective decomposes as
\begin{equation}
\mathbb{E}_{c}
\left[
D_{\mathrm{KL}}(p_\theta\|q_c)
\right]
=
\underbrace{
D_{\mathrm{KL}}(p_\theta\|q_{\mathrm{geo}})
}_{\text{distillation term}}
+
\underbrace{
(-\log Z)
}_{\text{conflict term}},
\label{eq:reverse_kl_decom_conclusion}
\end{equation}
where $-\log Z\ge 0$.
\end{proposition}

\begin{proof}
Expanding the reverse KL objective gives
\begin{equation}
\begin{aligned}
\mathcal{L}(\theta)
=
\mathbb{E}_{c}
\left[
\sum_{v\in\mathcal{V}}
p_\theta(v)\log\frac{p_\theta(v)}{q_c(v)}
\right]
=
\sum_{v\in\mathcal{V}}p_\theta(v)\log p_\theta(v)
-
\sum_{v\in\mathcal{V}}
p_\theta(v)\mathbb{E}_{c}[\log q_c(v)].
\end{aligned}
\label{eq:reverse_kl_decom}
\end{equation}
According to Eq.(\ref{eq:geometric-mean}),
\begin{equation}
\mathbb{E}_{c}[\log q_c(v)]
=
\log q_{\mathrm{geo}}(v)+\log Z.
\end{equation}
Substituting this relation into Eq.~(\ref{eq:reverse_kl_decom}) yields
\begin{equation}
\begin{aligned}
\mathcal{L}(\theta)
=
\sum_{v\in\mathcal{V}}
p_\theta(v)
\log\frac{p_\theta(v)}{q_{\mathrm{geo}}(v)}
-
\log Z
\sum_{v\in\mathcal{V}}p_\theta(v)
=
D_{\mathrm{KL}}(p_\theta\|q_{\mathrm{geo}})
-\log Z,
\end{aligned}
\end{equation}
which gives Eq.~(\ref{eq:reverse_kl_decom_conclusion}).

Moreover, by Jensen's inequality~\citep{jensen1906fonctions},
\begin{equation}
\begin{gathered}
Z
=
\sum_{v\in\mathcal{V}}
\exp\left(\mathbb{E}_{c}[\log q_c(v)]\right)
\le
\sum_{v\in\mathcal{V}}\mathbb{E}_{c}[q_c(v)]
=
1,
\end{gathered}
\end{equation}
Therefore, $-\log Z\ge 0$.
\end{proof}
Under reverse KL, the student is thus distilled toward the \textbf{normalized geometric mean} of context-conditioned teachers, rather than the arithmetic mean under forward KL.
The geometric mean assigns high probability only to tokens that receive consistently high probability across contexts.
Moreover, the conflict term $-\log Z$ reflects distributional disagreement among these teachers: lower distributional overlap decreases $Z$ and increases $-\log Z$.

\begin{corollary}[Gradient Independence of the Conflict Term]
\label{cor:conflict-gradient}
For fixed decoding histories, a fixed context distribution, and fixed context-conditioned teacher distributions, the conflict term $-\log Z$ is independent of the current student token distribution.
Therefore, it contributes no direct gradient with respect to the student parameters through $p_\theta$.
\end{corollary}

\begin{wrapfigure}{r}{0.48\linewidth}
\vspace{-2.9em}
\centering
    \includegraphics[width=\linewidth]{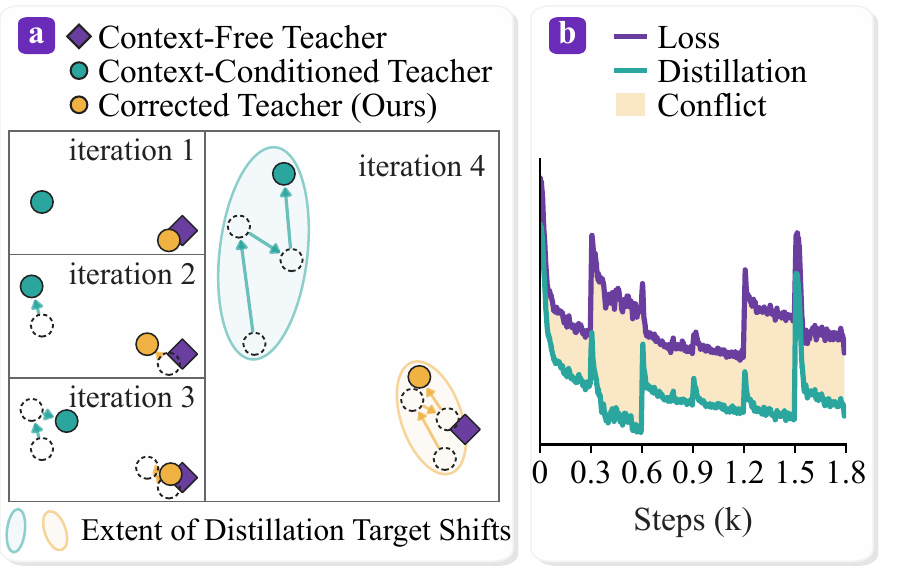}
\vspace{-2em}
\caption{
Characterizing training with evolving contexts using the reverse-KL decomposition.
(a) Evolving contexts cause shifts in context-conditioned teacher distributions.
(b) The conflict term varies with the sampled histories and contexts rather than through direct optimization.
}
\label{fig:decomposition}
\vspace{-2em}
\end{wrapfigure}

As in standard OPD, token-level optimization treats sampled histories as fixed; under this convention, only the distillation term is directly optimized, with the geometric mean serving as the target.
Applied across iterations, this fixed-history decomposition highlights target shifts and varying conflicts, motivating us to anchor the target to the context-free teacher and downweight corrections under high conflict.
Fig.~\ref{fig:decomposition}(a) visualizes the mean output embeddings of different distillation targets, showing that context-conditioned teachers undergo large shifts which may destabilize training, whereas the context-free teacher remains stable across updates.
Moreover, as shown in Fig.~\ref{fig:decomposition}(b), variations in the conflict value reflect changes in the sampled histories and context pools across iterations rather than direct optimization of the conflict term.

\begin{figure*}[t]
    \centering
    \includegraphics[width=1\linewidth]{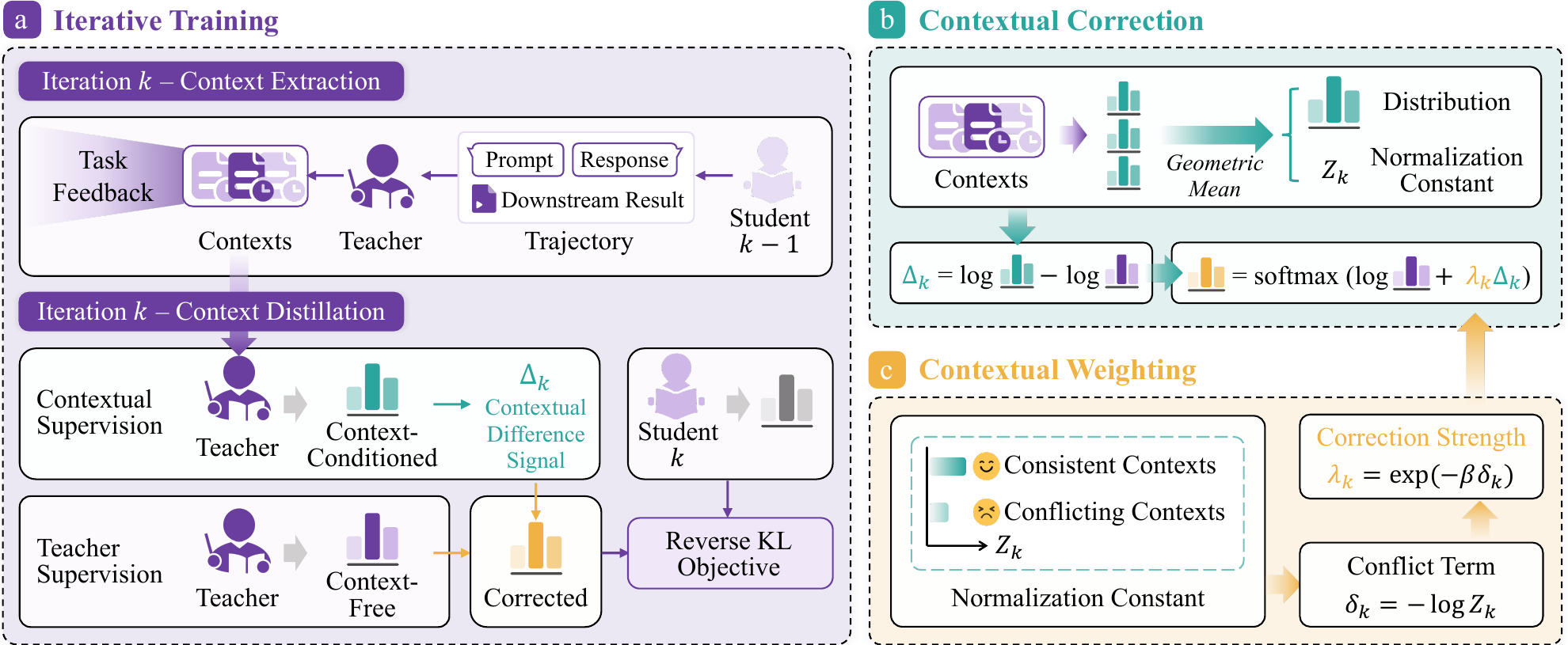}
    \vspace{-1.5em}
   \caption{
   Overview of Flux-OPD.
    (a) illustrates the \textit{iterative training} process, which divides a single training run into $K$ iterations, each consisting of context extraction and context distillation.
    (b) illustrates the \textit{contextual correction} strategy, which corrects the distillation target by combining the context-free teacher distribution with contextual difference signals from context-conditioned teachers.
    (c) illustrates the \textit{contextual weighting} strategy, which further adjusts correction strengths using the conflict term, strengthening corrections when context-conditioned teacher distributions are consistent and weakening them when they conflict.}
    \label{fig:method}
\end{figure*}

\section{Method}
\label{sec:method}

We propose Flux-OPD, an OPD paradigm that uses evolving contexts as
in-training supervision to incorporate contextual signals derived from task feedback.
As shown in Fig.~\ref{fig:method}, it is an iterative training process including two distillation strategies: contextual correction and contextual weighting.

\subsection{Iterative Training}
To adapt contextual supervision to the evolving student, we update contexts
with student performance.
We update contexts within a single training run, using a small set of trajectories at each update to provide timely feedback on the current student.
These evolving contexts make the training an iterative process that divides a single training run into $K$ iterations, each consisting of context extraction and context distillation.
As shown in Fig.~\ref{fig:method}(a), the teacher extracts contexts from student trajectories during context extraction and conditions on these contexts during context distillation.

\paragraph{Context Extraction}
Following existing context distillation methods~\citep{ye2026policy,ye2026online}, contexts consist of experience items extracted from student interaction trajectories.
At the beginning of iteration $k \in \{1,2,\dots,K\}$, we collect trajectories from the previous student policy $\pi_{\theta_{k-1}}$:
\begin{equation}
\mathcal{D}_{k-1}=\{(x_i,\hat{y}_i)\}_{i=1}^{n},
\end{equation}
where $x_i$ is the user prompt, $\hat{y}_i$ is the student response and optional supplements, such as downstream video generated from $\hat{y}_i$ in prompt optimization task and rubrics in medical question answering task.
For the first iteration, user prompts are randomly sampled from the entire training set.
For subsequent iterations, they are randomly sampled from the training data seen in the previous iteration.

Given each trajectory, the teacher extracts new experience items that differ from existing ones and accumulate them with previous experience items to form a complete context.
This process is repeated for $M$ turns with different random seeds, yielding a context pool:
\begin{equation}
\mathcal{C}_k=\{c_{k,1},\ldots,c_{k,M}\}.
\label{eq:M}
\end{equation}

\paragraph{Context Distillation}

Contexts sampled from the context pool $\mathcal{C}_k$ condition the teacher, allowing the target distribution to evolve with student performance.
Since the context-free teacher provides a stable anchor for the base target, while the conflict term indicates conflicting context-conditioned teacher supervision, we stabilize distillation through contextual correction and contextual weighting.

\subsection{Contextual Correction}
To stabilize the distillation target, the contextual correction strategy anchors it to the context-free teacher and injects contextual difference signals derived from evolving contexts, rather than directly using context-conditioned teachers as targets, as shown in Fig.~\ref{fig:method}(b).

For readability, we suppress the dependence of $q_{\mathrm{geo},k}$, $Z_k$, $\delta_k$, and $\lambda_k$ on the decoding history $h_t$ and position $t$.
For iteration $k$, at each response position, we first compute the
normalized geometric mean of context-conditioned teachers:
\begin{equation}
\begin{aligned}
\log \tilde{q}_{\mathrm{geo},k}(v)
&=
\frac{1}{|\mathcal{C}_k|}
\sum_{c \in \mathcal{C}_k} \log q_c(v),\\
q_{\mathrm{geo},k}(v)
&=
\frac{
\exp\!\left(\log \tilde{q}_{\mathrm{geo},k}(v)\right)
}{Z_k},
\end{aligned}
\end{equation}
where
\begin{equation}
Z_k
=
\sum_{v\in\mathcal{V}}
\exp(\log \tilde{q}_{\mathrm{geo},k}(v)).
\end{equation}
This is the empirical counterpart of Eq.(\ref{eq:geometric-mean}), using a uniform average over the finite context pool $\mathcal{C}_k$.
Since the context-conditioned teacher is constructed in log-probability space, we compute the contextual difference signal relative to the context-free teacher $q_0$, as
\begin{equation}
\Delta_k(v)
=
\log q_{\mathrm{geo},k}(v)
-
\log q_0(v).
\end{equation}
The final corrected teacher distribution is then obtained through log-space interpolation:
\begin{equation}
q^{\mathrm{flux}}_k(v)
=
\mathrm{softmax}_{v}
\left(
\log q_0(v)
+
\lambda_k \Delta_k(v)
\right),
\end{equation}
where $\lambda_k\in[0,1]$ controls the strength of contextual correction.
When $\lambda_k=0$, the target reduces to the context-free teacher $q_0$; when $\lambda_k=1$, it recovers the normalized geometric-mean teacher $q_{\mathrm{geo},k}$.
Thus, the context-free teacher $q_0$ provides a stable anchor, while the contextual difference signal $\Delta_k$ incorporates contextual information
derived from task feedback.

\subsection{Contextual Weighting}
\label{sec:weighting}
To downweight conflicting distributions, the contextual weighting strategy adjusts the correction strength using the conflict term, as shown in Fig.~\ref{fig:method}(c).

From the reverse KL decomposition in Eq.(\ref{eq:reverse_kl_decom_conclusion}), the conflict term measures distributional disagreement among context-conditioned teachers.
Motivated by this observation, we use it as a disagreement indicator to adjust the correction strength:
\begin{equation}
\delta_k=-\log Z_k,
\end{equation}
where a larger $\delta_k$ indicates lower overlap among context-conditioned teacher distributions at the current decoding position.

We design the weighting mechanism from a conflict-aware perspective: corrections are strengthened when context-conditioned teacher distributions are consistent and weakened when they conflict.
Specifically, we use a monotonically decreasing mapping: 
\begin{equation}
 \lambda_k = \exp(-\beta \delta_k),
\label{eq:weight}
\end{equation}
where $\beta$ is a scaling factor. At each decoding position, the same weight is applied to the contextual correction across all vocabulary logits.

\subsection{Training Objective}
In summary, Flux-OPD trains the student toward the reformulated target at each iteration by computing the reverse KL divergence between the student distribution and the corrected teacher distribution:
\begin{equation}
\mathcal{L}^{\mathrm{flux}}_k
=
\mathbb{E}_{h_t}
\left[
D_{\mathrm{KL}}
\left(
p_{\theta}(\cdot\mid h_t)
\|
q^{\mathrm{flux}}_k(\cdot\mid h_t)
\right)
\right].
\end{equation}

\section{Experiments}

\subsection{Implementation Details}

\paragraph{Models}
We conduct experiments under different student-teacher settings: Qwen3-VL-Instruct~\citep{yang2025qwen3} with a 2B student and an 8B teacher, Qwen3-VL-Instruct with a 4B student and an 8B teacher, and Qwen3 with a 1.7B student and an 8B teacher.
For each setting, the student is initialized from an SFT model trained for 5 epochs on offline teacher-distilled outputs, while the teacher uses the base model. All models operate in non-thinking mode.

\paragraph{Baselines}
We compare Flux-OPD with several OPD paradigms: (1) OPD~\citep{agarwal2024policy}, which distills teacher supervision on trajectories sampled from current student policy; (2) OPCD~\citep{ye2026policy}, which conditions the teacher on privileged experience items as contexts; and (3) OEL~\citep{ye2026online}, which equips the teacher with evolving contexts for continuous student improvement. 
Moreover, on tasks with available rubrics, we compare RaR~\citep{gunjal2025rubrics} with OPD paradigms by matching the number of student-generated completions, using the teacher as rubric judge (Appendix~\ref{app:compute_budget}).
All baselines are implemented in a single-run training setting;
we denote the resulting OEL adaptation as OEL$^\dagger$.

\paragraph{Tasks}
We focus on open-ended tasks with real-world applicability: prompt optimization for video generation and medical question answering, both involving complex preferences and requiring rich contextual information.
Prompt optimization rewrites imperfect user prompts into generator-aligned prompts, requiring not only text-level teacher supervision but also contexts that convey video-level preferences of a specific generator.
Medical question answering spans diverse user queries and clinical workflows, where contexts help capture domain knowledge and interaction preferences.

\paragraph{Training}
(1) Datasets.
We use 10K SFT prompts from VPO~\citep{cheng2025vpo} for prompt optimization and RaR-Medicine~\citep{gunjal2025rubrics} with 18K rubric-paired questions for medical question answering.
(2) Context Extraction.
Each context is extracted from $8$ student rollouts with $M=3$ random seeds per iteration.
For prompt optimization, downstream videos generated from student responses serve as supplements, while paired rubrics serve as supplements for medical question answering.
(3) Context Distillation.
OPCD and OEL$^\dagger$ use the same initial contexts as Flux-OPD.
OEL$^\dagger$ and Flux-OPD update contexts every $300$ steps.

\begin{table*}[t]
  \centering
    \caption{Results on the prompt optimization task with a Qwen3-VL-Instruct-8B teacher. The best and second-best student results are highlighted in bold and underlined, respectively.}
  \label{tab:prompt_optimization}

  {
  \fontsize{8pt}{\baselineskip}\selectfont
  \rmfamily
  \setlength{\tabcolsep}{1mm}
  
  \begin{tabular}{
    lcccccc
  }
    \toprule
    \multirow{2}[2]{*}{\textbf{Method}}
    & \multicolumn{3}{c}{\textbf{VBench~$\boldsymbol{\uparrow}$}}
    & \multicolumn{3}{c}{\textbf{VideoVerse~$\boldsymbol{\uparrow}$}} \\
    \cmidrule(lr){2-4}
    \cmidrule(lr){5-7}

    & \textbf{VACE}
    & \textbf{CogVideoX}
    & \textbf{Average}
    & \textbf{VACE}
    & \textbf{CogVideoX}
    & \textbf{Average} \\
    \midrule

    Teacher
    & 81.48
    & 77.68
    & 79.58
    & 1105
    & 1123
    & 1114.0\\
    \midrule

    \rowcolor{gray!10}
    \multicolumn{7}{c}{\textbf{\textit{Qwen3-VL-Instruct-2B}}} \\
    \addlinespace[2pt]
    Student
    & 80.86
    & 76.48
    & 78.67
    & 1092
    & 1099
    & 1095.5\\

    + OPD
    & \score{\underline{81.25}}{\pos{+0.39}}
    & \score{\underline{76.45}}{\nega{-0.03}}
    & \score{\underline{78.85}}{\pos{+0.18}}
    & \score{1106}{\pos{+14}}
    & \score{\underline{1097}}{\nega{-2}}
    & \score{\underline{1101.5}}{\pos{+6}}\\

    + OPCD
    & \score{81.06}{\pos{+0.20}}
    & \score{76.07}{\nega{-0.41}}
    & \score{78.57}{\nega{-0.10}}
    & \score{\underline{1107}}{\pos{+15}}
    & \score{1078}{\nega{-21}}
    & \score{1092.5}{\nega{-3}}\\

    + OEL$^\dagger$
    & \score{79.97}{\nega{-0.89}}
    & \score{74.30}{\nega{-2.18}}
    & \score{77.14}{\nega{-1.53}}
    & \score{1106}{\pos{+14}}
    & \score{1022}{\nega{-77}}
    & \score{1064.0}{\nega{-31.5}}\\

    \rowcolor{myLightPurple}
    + Flux-OPD
    & \score{\textbf{81.50}}{\pos{+0.64}}
    & \score{\textbf{77.10}}{\pos{+0.62}}
    & \score{\textbf{79.30}}{\pos{+0.63}}
    & \score{\textbf{1116}}{\pos{+24}}
    & \score{\textbf{1106}}{\pos{+7}}
    & \score{\textbf{1111.0}}{\pos{+15.5}}\\
    \midrule

    \rowcolor{gray!10}
    \multicolumn{7}{c}{\textbf{\textit{Qwen3-VL-Instruct-4B}}} \\
    \addlinespace[2pt]
    Student
    & 81.21
    & 77.74
    & 79.48
    & 1107
    & 1122
    & 1114.5 \\

    + OPD
    & \score{81.02}{\nega{-0.19}}
    & \score{\underline{77.75}}{\pos{+0.01}}
    & \score{\underline{79.39}}{\nega{-0.09}}
    & \score{1081}{\nega{-26}}
    & \score{\underline{1119}}{\nega{-3}}
    & \score{1100.0}{\nega{-14.5}} \\

    + OPCD
    & \score{\underline{81.09}}{\nega{-0.12}}
    & \score{77.49}{\nega{-0.25}}
    & \score{79.29}{\nega{-0.19}}
    & \score{\underline{1122}}{\pos{+15}}
    & \score{1113}{\nega{-9}}
    & \score{\underline{1117.5}}{\pos{+3}} \\

    + OEL$^\dagger$
    & \score{79.77}{\nega{-1.44}}
    & \score{71.09}{\nega{-6.65}}
    & \score{75.43}{\nega{-4.05}}
    & \score{1103}{\nega{-4}}
    & \score{1041}{\nega{-81}}
    & \score{1072.0}{\nega{-42.5}}\\

    \rowcolor{myLightPurple}
    + Flux-OPD
    & \score{\textbf{81.91}}{\pos{+0.70}}
    & \score{\textbf{77.95}}{\pos{+0.21}}
    & \score{\textbf{79.93}}{\pos{+0.45}}
    & \score{\textbf{1128}}{\pos{+21}}
    & \score{\textbf{1122}}{\pos{+0}}
    & \score{\textbf{1125.0}}{\pos{+10.5}} \\
    \bottomrule
  \end{tabular}
  }
\end{table*}

\paragraph{Evaluation}
(1) Prompt Optimization.
We evaluate all methods with two downstream video generators from different model families, Wan2.1-VACE-1.3B and CogVideoX-2B~\citep{jiang2025vace,yang2025cogvideox}, on two video generation benchmarks: VBench and VideoVerse~\citep{huang2024vbench,wang2025videoverse}.
(2) Medical Question Answering.
We evaluate on HealthBench~\citep{arora2025healthbench} to assess model safety and helpfulness in realistic medical scenarios.

\begin{wraptable}{r}{0.45\linewidth}
\vspace{-2.25em}

\centering

\caption{Results on the medical question answering task with a Qwen3 8B teacher.
The best and second-best student results are highlighted in bold and underlined, respectively.}
\label{tab:medicine}

\vspace{0em}

{
  \fontsize{8pt}{\baselineskip}\selectfont
  \rmfamily
  \setlength{\tabcolsep}{1mm}
\begin{tabular}{lcc}
    \toprule
    \multirow{2}[2]{*}{\textbf{Method}}
    & \multicolumn{2}{c}{\textbf{HealthBench}~$\boldsymbol{\uparrow}$} \\
    \cmidrule(lr){2-3}
    & \textbf{w/o Rubrics} & \textbf{w/ Rubrics} \\
    \midrule
    
    Teacher
    & 37.38  & - \\
    \midrule
    
    \rowcolor{gray!10}
    \multicolumn{3}{c}{\textbf{\textit{Qwen3 1.7B}}} \\
    \addlinespace[2pt]
    
    Student
    & 19.06 & - \\
    
    + OPD
    & \score{19.63}{\pos{+0.57}} & - \\
    
    + OPCD
    & \score{19.29}{\pos{+0.23}} & \score{19.53}{\pos{+0.47}} \\
    
    + OEL$^\dagger$
    & \score{\underline{19.69}}{\pos{+0.63}}  & \score{\underline{19.66}}{\pos{+0.60}} \\
    
    + RaR
    & - & \score{16.91}{\nega{-2.15}} \\
    
    \rowcolor{myLightPurple}
    + Flux-OPD
    & \score{\textbf{20.20}}{\pos{+0.82}}
    & \score{\textbf{20.26}}{\pos{+1.20}}  \\
    \bottomrule
    \end{tabular}
}

\vspace{-1.5em}
\end{wraptable}

\subsection{Main Results}

As shown in Tabs.~\ref{tab:prompt_optimization} and~\ref{tab:medicine}, Flux-OPD achieves the best reported overall performance across the evaluated settings on two open-ended tasks.
In medical question answering, where task preferences are primarily text-level, other OPD paradigms improve over the initial policy.
While in prompt optimization, where video-level feedback provides
generator-specific downstream signals, they may degrade the
initial student policy, suggesting limited ability to effectively leverage such contexts.
RaR also fails to improve under matched student-completion budget.

Moreover, we conduct three independent training runs for OPD and Flux-OPD. As shown in Fig.~\ref{fig:sft_epoch} (Left), Flux-OPD achieves higher mean scores than OPD on VBench (CogVideoX) and HealthBench, with comparable variation across runs.

\begin{figure}[t]
\centering
\includegraphics[width=1\linewidth]{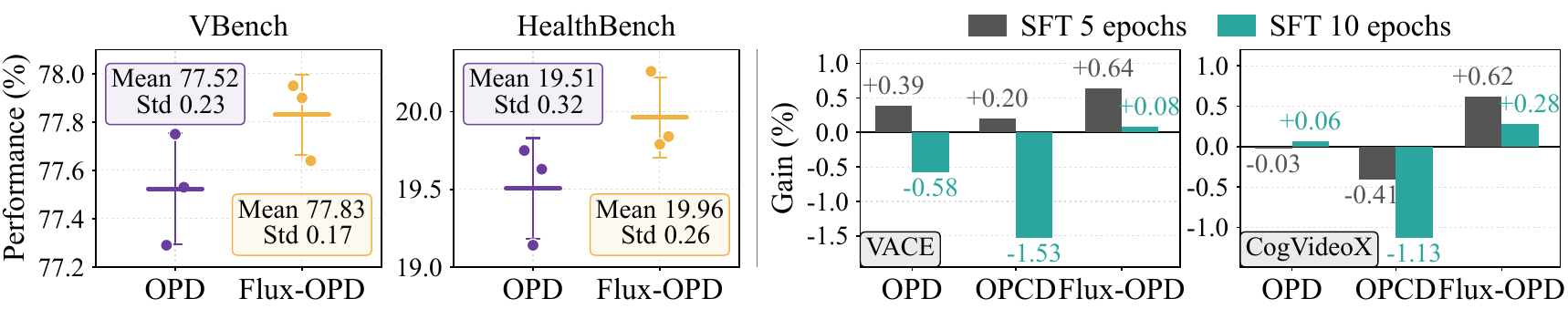}
\vspace{-2em}
\caption{Left: Results across three independent training runs. Flux-OPD achieves higher mean scores than OPD. Right: Effect of student initialization. While stronger initialization reduces the gains from OPD and OPCD, Flux-OPD maintains performance improvements.}
\label{fig:sft_epoch}
\end{figure}

Moreover, we find that degradation is more common for the 4B student than for the 2B student.
Since 5-epoch SFT may not fully absorb basic teacher knowledge, subsequent distillation gains may come from either general teacher supervision or contextual information.
To isolate the effect of contextual supervision, we further evaluate OPD methods with a stronger 10-epoch SFT initialization.
As shown in Fig.~\ref{fig:sft_epoch} (Right), stronger initialization generally reduces the gains of OPD and OPCD, with degradation particularly evident for OPCD. In contrast, Flux-OPD improves both initializations, indicating greater robustness.
These results suggest that as students better absorb teacher knowledge, direct distribution matching provides diminishing benefits and may introduce harmful corrections, especially with context-conditioned teachers.

\subsection{Analysis}

We conduct additional analyses on VBench (4B student, 8B teacher) and HealthBench.

\paragraph{Effects of Evolving Contexts}
As shown in Tab.~\ref{tab:abl}, full Flux-OPD achieves higher scores than V1, the variant w/o evolving contexts, demonstrating that contexts evolving with student performance improve the student more effectively under the same number of training steps.

\paragraph{Effects of Contextual Correction}
Contextual correction combines the context-free teacher with difference signals from context-conditioned teachers.
As shown in Tab.~\ref{tab:abl}, both OEL$^\dagger$ and V2 have evolving contexts, while V2, equipped with contextual correction, achieves higher scores than OEL$^\dagger$ under all three correction strengths.
This demonstrates the effectiveness of using context-conditioned teachers as contextual difference signals rather than direct distillation targets.

\paragraph{Effects of Contextual Weighting}
Contextual weighting adjusts correction strengths based on the conflict term in Eq.(\ref{eq:weight}).
As shown in Tab.~\ref{tab:abl}, full Flux-OPD achieves higher scores than all V2 variants with static strengths, showing that the weighting strategy is more effective than using a fixed strength.
Moreover, it outperforms entropy- and JSD-based weighting under matched mean correction strength, supporting the use of the reverse-KL-derived conflict score (Appendix Sec.~\ref{app:alternative_weighting}).

\paragraph{Effect of the Scaling Factor}
\label{sec:param}

The scaling factor $\beta$ controls contextual correction based on conflicts among context-conditioned teachers.
Prompt optimization involves additional variability from downstream video generation, whereas medical question answering relies on text-only interactions and exhibits more consistent contextual supervision (Appendix Fig.~\ref{fig:supp_conflict}).
Since more consistent contexts yield smaller conflict values and thus stronger corrections under the same $\beta$, a larger $\beta$ can more strongly suppress contextual effects.
Consistent with this observation, Tab.~\ref{tab:param_conflict} shows that $\beta=50$ gives the highest HealthBench score.
Moreover, within prompt optimization, $\beta=8$ performs better on CogVideoX by retaining stronger contextual correction.
This is consistent with the observation that the initial student is already competitive with the context-free teacher.

\begin{table*}[t]
  \centering

  \begin{minipage}[t]{0.58\textwidth}
  \centering
    \caption{Ablation study of Flux-OPD. $\lambda_k$ denotes the correction strength.}
  \label{tab:abl}

  {
  \fontsize{8pt}{\baselineskip}\selectfont
  \rmfamily
  \setlength{\tabcolsep}{0.6mm}

  \begin{tabular}{lcccccc}
    \toprule
    \textbf{Variant}
    & \makecell{\textbf{Evolving}\\\textbf{Contexts}}
    & \makecell{\textbf{Contextual}\\\textbf{Correction}}
    & \makecell{\textbf{Contextual}\\\textbf{Weighting}}
    & $\boldsymbol{\lambda_k}$
    & \makecell{\textbf{VBench}\\\textbf{(VACE)}}
    & \makecell{\textbf{Health-}\\\textbf{Bench}}
    \\
    \midrule

    V1
    & \icono
    & \icoyes
    & \icoyes
    & Eq.(\ref{eq:weight})
    & 81.29 
    & 19.63\\
    \midrule

    OEL$^\dagger$
    & \icoyes
    & \icono
    & \icono
    & -
    & 79.77 
    & 19.66\\
    \midrule

    \multirow{3}{*}{V2}
    & \multirow{3}{*}{\icoyes}
    & \multirow{3}{*}{\icoyes}
    & \multirow{3}{*}{\icono}
    & $\equiv 0.9$
    & 81.74 
    & 19.73\\

    & & &
    & $\equiv 0.7$
    & 81.41 
    & 20.03\\

    & & &
    & $\equiv 0.5$
    & 80.93 
    & 19.97\\
    \midrule

    \rowcolor{myLightPurple}
    Full
    & \icoyes
    & \icoyes
    & \icoyes
    & Eq.(\ref{eq:weight})
    & \textbf{81.91} 
    & \textbf{20.26}\\
    \bottomrule
  \end{tabular}
  }
  \end{minipage}%
    \hfill%
  \begin{minipage}[t]{0.39\textwidth}
  \centering
  \caption{Effect of the scaling factor $\beta$ in Flux-OPD.}
  \label{tab:param_conflict}

  \vspace{1.3em}

  {
  \fontsize{8pt}{\baselineskip}\selectfont
  \rmfamily
  \setlength{\tabcolsep}{0.8mm}

  \begin{tabular}{lcccc}
    \toprule
    \multirow{2}[2]{*}{\textbf{Method}} 
    & \multirow{2}[2]{*}{$\boldsymbol{\beta}$} 
    & \multicolumn{2}{c}{\textbf{VBench}} & \multirow{2}[2]{*}{\makecell{\textbf{Health-}\\\textbf{Bench}}}\\
    \cmidrule(lr){3-4}
    &  & \textbf{VACE}  & \textbf{CogVideoX} &  \\
    \midrule
    Teacher & - &  81.48      & 77.68 & 37.38   \\
    \midrule
    Student & - &  81.21      & 77.74 & 19.06   \\
    \midrule
    \multirow{3}[0]{*}{Flux-OPD} 
        & 8   & 81.27 & \cellcolor{myLightPurple}\textbf{77.95} & 19.27 \\
      & 30   & \cellcolor{myLightPurple}\textbf{81.91} & 77.44 & 19.33 \\
      & 50   & 81.49 & 77.54 & \cellcolor{myLightPurple}\textbf{20.26} \\
    \bottomrule
  \end{tabular}
  }
  \end{minipage}

\end{table*}

\subsection{Discussion}
\vspace{-1em}
\begin{figure*}[h]
    \centering
    \includegraphics[width=\linewidth]{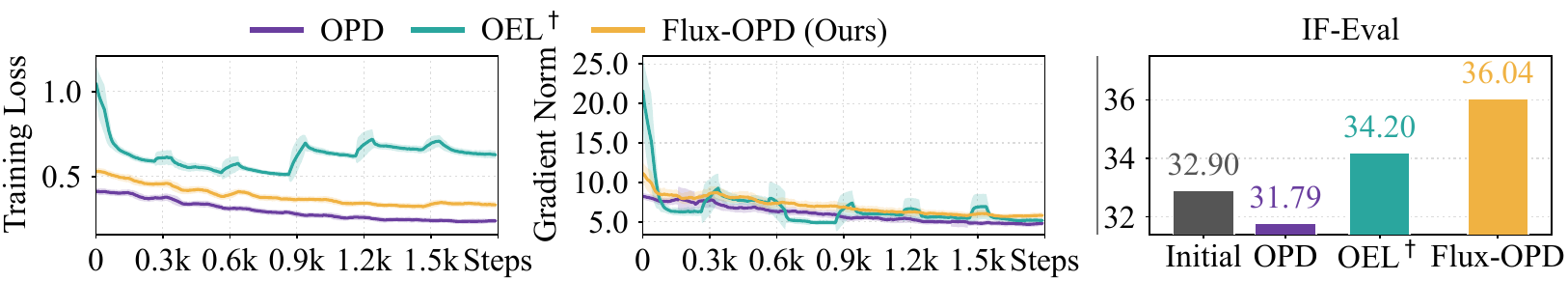}
    \vspace{-1.5em}
    \caption{
    Left: Training stability. The training losses of OPD and Flux-OPD decrease steadily, whereas OEL$^\dagger$ exhibits loss surges. Flux-OPD also maintains a stable gradient norm.
    Right: Out-of-distribution performance on IF-Eval. Flux-OPD achieves the highest prompt-level strict accuracy.
}
    \label{fig:stability}
\end{figure*}

\noindent\textbf{Stability}
As shown in Fig.~\ref{fig:stability} (Left), Flux-OPD exhibits stable training dynamics similar to OPD, while OEL$^\dagger$ shows abrupt loss surges and larger gradient norm fluctuations around context updates.

\noindent\textbf{Generalization}
We evaluate out-of-distribution instruction following on IF-Eval~\citep{zhou2023instructionfollowingevaluationlargelanguage} after training on medical question answering.
As shown in Fig.~\ref{fig:stability} (Right), Flux-OPD achieves higher prompt-level strict accuracy than OPD, suggesting better cross-domain generalization.

\section{Conclusion}
We propose Flux-OPD, an OPD paradigm that uses evolving contexts as in-training supervision to incorporate contextual signals derived from task feedback.
We characterize context-conditioned distillation through a reverse-KL decomposition and address distillation-target shifts and disagreement among context-conditioned teachers by anchoring distillation to the context-free teacher, injecting contextual difference signals, and weighting them according to teacher disagreement.
These results highlight the potential to combine teacher supervision with evolving contexts.

\subsection*{AI use statement}

In this work, we used generative AI tools exclusively for grammatical error correction and language polishing. All AI-assisted edits were reviewed by the authors, who take full responsibility for the final content of this work.

\subsection*{Reproducibility Statement}
We provide additional details to facilitate reproducibility in the appendix, including the implementation and training setup, dataset and evaluation details, hyperparameter configurations, and supplementary analyses. The derivations and additional explanations supporting our method are also provided in the appendix.

\bibliography{iclr2027_conference}
\bibliographystyle{iclr2027_conference}

\clearpage
\appendix

\section*{Appendix}

\section{Prompt Template}
All prompts are intentionally concise and require \textbf{minimal human engineering}. They contain no human-crafted domain knowledge, allowing the models to rely on their own capabilities.

\subsection{Context Extraction Prompt}

Figs.~\ref{fig:prompt_context_extraction_pe} and~\ref{fig:prompt_context_extraction_medical} show the context extraction prompts used by OPCD, OEL$^\dagger$, and Flux-OPD for the prompt optimization and medical question answering tasks. These prompts extract new experience items beyond the existing ones.

\subsection{Context Distillation Prompt}

Figs.~\ref{fig:prompt_context_distillation_pe} and~\ref{fig:prompt_context_distillation_mqa} show the in-context learning prompts used by context-conditioned teachers in OPCD, OEL$^\dagger$, and Flux-OPD for the prompt optimization and medical question answering tasks.

\subsection{System Prompt}
Fig.~\ref{fig:prompt_system_pe} shows the system prompt used for offline SFT data construction, student training, and evaluation in the prompt optimization task. We do not use an explicit system prompt for the medical question answering task.

\section{Context Example}
Figs.~\ref{fig:context_pe} and~\ref{fig:context_mqa} show context examples from OPCD, OEL$^\dagger$, and Flux-OPD for the prompt optimization and medical question answering tasks.

\section{Additional Training Details}
\subsection{Hyperparameters}
\paragraph{General Hyperparameters}
For each training paradigm, we use a global batch size of $64$ and a learning rate of $2\times10^{-6}$.
SFT runs for $5$ epochs, OPD-style methods run for $10$ epochs, and RaR runs for $1.25$ epochs with $8$ completions per prompt.
For computational efficiency, the implemented distillation loss retains the top-$64$ target logits and renormalizes them over the selected tokens.
In Flux-OPD, this top-$k$ truncation is applied only to the final distillation loss computation.
The conflict term and the corresponding weighting signal are computed from full-vocabulary log-softmax distributions.

\paragraph{Hyperparameters of Flux-OPD}
We list the hyperparameters of Flux-OPD in Tab.~\ref{tab:supp_hyper}.
We use a smaller scaling factor for prompt optimization ($\beta=30$) than for medical question answering ($\beta=50$).
As shown in Fig.~\ref{fig:supp_conflict}, prompt optimization exhibits higher conflict values, partly due to the additional variability introduced by downstream video generation.
Since larger conflict values already suppress contextual correction through Eq.(\ref{eq:weight}), a smaller $\beta$ avoids overly suppressing contextual effects.

\begin{table*}[h]
\centering
{
\fontsize{9pt}{\baselineskip}\selectfont
\rmfamily
\setlength{\tabcolsep}{1mm}
\caption{Additional hyperparameters of Flux-OPD.}
\label{tab:supp_hyper}
\begin{tabular}{cccc}
\toprule
\textbf{Task} & \textbf{Student \& Teacher} & \textbf{Downstream Model} & $\boldsymbol{\beta}$  \\
\midrule
\multirow{4}[4]{*}{Prompt Optimization} & \multirow{2}[2]{*}{\makecell{Qwen3-VL-Instruct\\8B$\rightarrow$2B}} & Wan2.1-VACE-1.3B & 30    \\

\cmidrule{3-4}
      &       & CogVideoX-2B & 30    \\

\cmidrule{2-4}          & \multirow{2}[2]{*}{\makecell{Qwen3-VL-Instruct\\8B$\rightarrow$4B}} & Wan2.1-VACE-1.3B & 30    \\

\cmidrule{3-4}
      &       & CogVideoX-2B & 8    \\
\midrule
Medical Question Answering & Qwen3 8B$\rightarrow$1.7B & -     & 50    \\
\bottomrule
\end{tabular}}

\end{table*}

\subsection{Context}
For each of the $M=3$ random seeds, the teacher sequentially processes the same $n=8$ trajectories. At each trajectory, it receives the experience items accumulated so far together with the user prompt, student response, and task-specific supplement (a generated video or medical rubrics), and is instructed to append task-general items not already covered. Each extraction call is capped at $256$ new tokens; under this output specification, each context contains at most 2,048 generated tokens. Generation reaching this limit retains its prefix, and the accumulated context is not otherwise truncated. Deduplication is semantic and prompt-based, with no post-hoc lexical or embedding-based filtering. Contexts are rebuilt at each update rather than accumulated across iterations. Every distillation step uses all $M=3$ contexts in the current pool, without subset sampling.

\subsection{Compute and Budget Accounting}
\label{app:compute_budget}
A teacher forward denotes one batched pass over the $64$ student completions in an optimization step.
OPD uses one context-free teacher forward per step, whereas Flux-OPD uses one context-free and $M=3$ context-conditioned forwards, for four teacher forwards in total.
Teacher logits are computed online and are not cached across optimization steps; only the extracted contexts are reused for $300$ steps.
At initialization and each context update, three contexts are extracted from eight student rollouts, requiring $3\times8=24$ teacher generation calls, each capped at $256$ new tokens.
The initial context pool is constructed once and shared across the context-based methods.
For prompt optimization, the eight auxiliary videos are generated once per update and reused across the three extraction seeds.
For medical question answering, the $10$-epoch OPD-style runs generate one completion per prompt per epoch, while RaR generates $8$ completions per prompt for $1.25$ epochs.
The executed runs contain $177{,}920$ and $177{,}472$ student completions, respectively, a $0.25\%$ difference due to batch rounding.
Thus, the comparison matches student-completion counts, not teacher-forward compute or total training FLOPs.

\subsection{Effect of the Number of Contexts}
Flux-OPD uses one context-free and $M$ context-conditioned teacher forwards per step. As shown in Tab.~\ref{tab:context_number}, reducing $M$ from $3$ to $2$ decreases the number of teacher forwards from $4$ to $3$, while changing VBench by only $+0.06$ on VACE and $-0.20$ on CogVideoX.
Performance changes only marginally from $M=2$ to $M=3$, showing diminishing returns from additional context-conditioned teacher forwards.
The larger degradation at $M=1$ indicates the importance of using multiple contexts to estimate inter-context disagreement and apply contextual weighting.

\begin{table}[h]
\centering
\caption{
Effect of the number of contexts $M$ with a Qwen3-VL-Instruct-8B teacher.
Flux-OPD uses one context-free and $M$ context-conditioned teacher forwards per optimization step.
}
\label{tab:context_number}

{
\fontsize{9pt}{\baselineskip}\selectfont
\rmfamily
\setlength{\tabcolsep}{1.5mm}

\begin{tabular}{cccc}
\toprule
\multirow{2}[5]{*}{$\boldsymbol{M}$}
& \multirow{2}[5]{*}{\makecell{\textbf{Teacher Forwards}\\\textbf{per Step}}}
& \multicolumn{2}{c}{\textbf{VBench} $\boldsymbol{\uparrow}$} \\
\cmidrule(lr){3-4}
& & \makecell{\textbf{VACE}\\\textbf{(4B Student)}}
  & \makecell{\textbf{CogVideoX}\\\textbf{(2B Student)}} \\
\midrule
1 & 2 & 77.84 & 75.62 \\
2 & 3 & \textbf{81.97} & \underline{76.90} \\
3 & 4 & \underline{81.91} & \textbf{77.10} \\
\bottomrule
\end{tabular}
}
\end{table}

\subsection{Conflict Dynamics}
We use the conflict metric to monitor distributional disagreement among context-conditioned teachers during training.
For each response token position, the teacher is queried under multiple sampled contexts, producing multiple context-conditioned token distributions.
These distributions are combined through a geometric mean in log-probability space.
A larger conflict value indicates lower distributional overlap and stronger disagreement among context-conditioned teachers.
We report the average conflict over valid response tokens at each training step.

As shown in Fig.~\ref{fig:supp_conflict}, medical question answering exhibits lower conflict values than prompt optimization, indicating more consistent contextual supervision.
The higher conflict values in prompt optimization may result from the additional variability introduced by downstream video generation, which can reduce the consistency of the resulting contexts.

\subsection{Alternative Weighting Criteria}
\label{app:alternative_weighting}

To compare our reverse-KL-derived conflict score with generic weighting criteria, we consider two alternatives.
Let $\{p_k^{(m)}\}_{m=1}^{M}$ denote the context-conditioned teacher distributions at token position $k$. We define the entropy and Jensen-Shannon divergence (JSD) scores as
\begin{equation}
s_k^{\mathrm{Ent}}
=
\frac{1}{M}\sum_{m=1}^{M} H\!\left(p_k^{(m)}\right),
\end{equation}
and
\begin{equation}
s_k^{\mathrm{JSD}}
=
H\!\left(\bar p_k\right)
-
\frac{1}{M}\sum_{m=1}^{M}H\!\left(p_k^{(m)}\right),
\qquad
\bar p_k=\frac{1}{M}\sum_{m=1}^{M}p_k^{(m)}.
\end{equation}
where $H(p)=-\sum_v p(v)\log p(v)$.
The entropy score measures the uncertainty of context-conditioned teachers, while the JSD score measures their distributional disagreement.
Each score is mapped to a contextual weight by
\begin{equation}
\lambda_k^{r}=\exp(-\alpha_r s_k^{r}),
\qquad
r\in\{\mathrm{Ent},\mathrm{JSD}\}.
\end{equation}

Because these scores have different numerical scales, we avoid downstream tuning and choose $\alpha_r$ for each training micro-batch such that the mean weight over valid response tokens matches that produced by Flux-OPD:
\begin{equation}
\frac{1}{|\mathcal{B}|}\sum_{k\in\mathcal{B}}
\exp(-\alpha_r s_k^r)
=
\frac{1}{|\mathcal{B}|}\sum_{k\in\mathcal{B}}
\exp(-\beta\delta_k),
\end{equation}
where $\mathcal{B}$ denotes the set of valid response-token positions in the micro-batch.
All other training and evaluation settings are unchanged, so the comparison keeps the mean contextual weight matched while varying the weighting criterion.
As shown in Tab.~\ref{tab:supp_weighting}, the reverse-KL-derived conflict score outperforms entropy and JSD weighting by 0.55 and 0.15 points, respectively. This result suggests that the reverse-KL-derived conflict score is more effective than these generic uncertainty- and disagreement-based alternatives.

\begin{table}[h]
\centering
\caption{Comparison of token-level weighting criteria for Qwen3-VL-Instruct 8B$\rightarrow$2B, evaluated on VBench with CogVideoX-2B. For each baseline, the mean contextual weight is matched to that of Flux-OPD within every training micro-batch.}
\label{tab:supp_weighting}
{
\fontsize{9pt}{\baselineskip}\selectfont
\rmfamily
\setlength{\tabcolsep}{3mm}
\begin{tabular}{lc}
\toprule
\textbf{Weighting Criterion} & \textbf{VBench~$\boldsymbol{\uparrow}$} \\
\midrule
Entropy              & 76.55 \\
JSD                  & \underline{76.95} \\
Reverse-KL Conflict (Ours) & \textbf{77.10} \\
\bottomrule
\end{tabular}
}
\end{table}

\section{Additional Evaluation Details}
For the prompt optimization task, we evaluate the 10 VBench dimensions~\citep{huang2024vbench} reported in prior works~\citep{cheng2025vpo,ji2025prompt}.
The quality dimensions include Subject Consistency, Background Consistency, Motion Smoothness, Dynamic Degree, Aesthetic Quality, and Image Quality, while the semantic dimensions include Multiple Objects, Human Action, Scene, and Appearance Style. These dimensions cover 523 prompts in total.
We evaluate all dimensions of the full VideoVerse~\citep{wang2025videoverse}, covering 300 prompts.
VBench reports scores based on automatic metrics, whereas VideoVerse uses Qwen2.5-VL-32B-Instruct for scoring following the official protocol.

\begin{figure*}[t]
\centering
\includegraphics[width=1\linewidth]{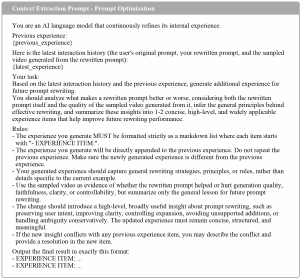}
\caption{Prompt for context extraction in the prompt optimization task.}
\label{fig:prompt_context_extraction_pe}
\end{figure*}

\begin{figure*}[t]
\centering
\includegraphics[width=1\linewidth]{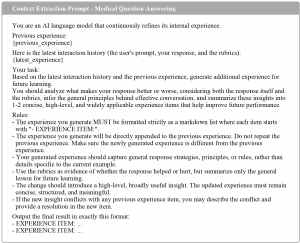}
\caption{Prompt for context extraction in the medical question answering task.}
\label{fig:prompt_context_extraction_medical}
\end{figure*}

\begin{figure*}[t]
\centering
\includegraphics[width=1\linewidth]{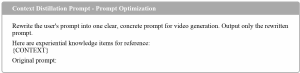}
\caption{Prompt for context distillation in the prompt optimization task.}
\label{fig:prompt_context_distillation_pe}
\end{figure*}

\begin{figure*}[t]
\centering
\includegraphics[width=1\linewidth]{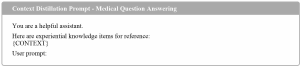}
\caption{Prompt for context distillation in the medical question answering task.}
\label{fig:prompt_context_distillation_mqa}
\end{figure*}

\begin{figure*}[t]
\centering
\includegraphics[width=1\linewidth]{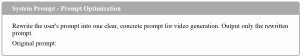}
\caption{System prompt for the prompt optimization task.}
\label{fig:prompt_system_pe}
\end{figure*}

\begin{figure*}[t]
\centering
\includegraphics[width=1\linewidth]{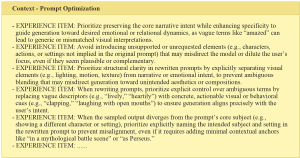}
\caption{Context example for the prompt optimization task.}
\label{fig:context_pe}
\end{figure*}

\begin{figure*}[t]
\centering
\includegraphics[width=1\linewidth]{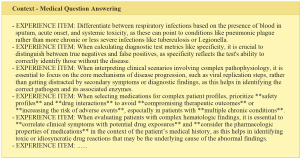}
\caption{Context example for the medical question answering task.}
\label{fig:context_mqa}
\end{figure*}

\begin{figure*}[t]
    \centering
    \includegraphics[width=0.9\linewidth]{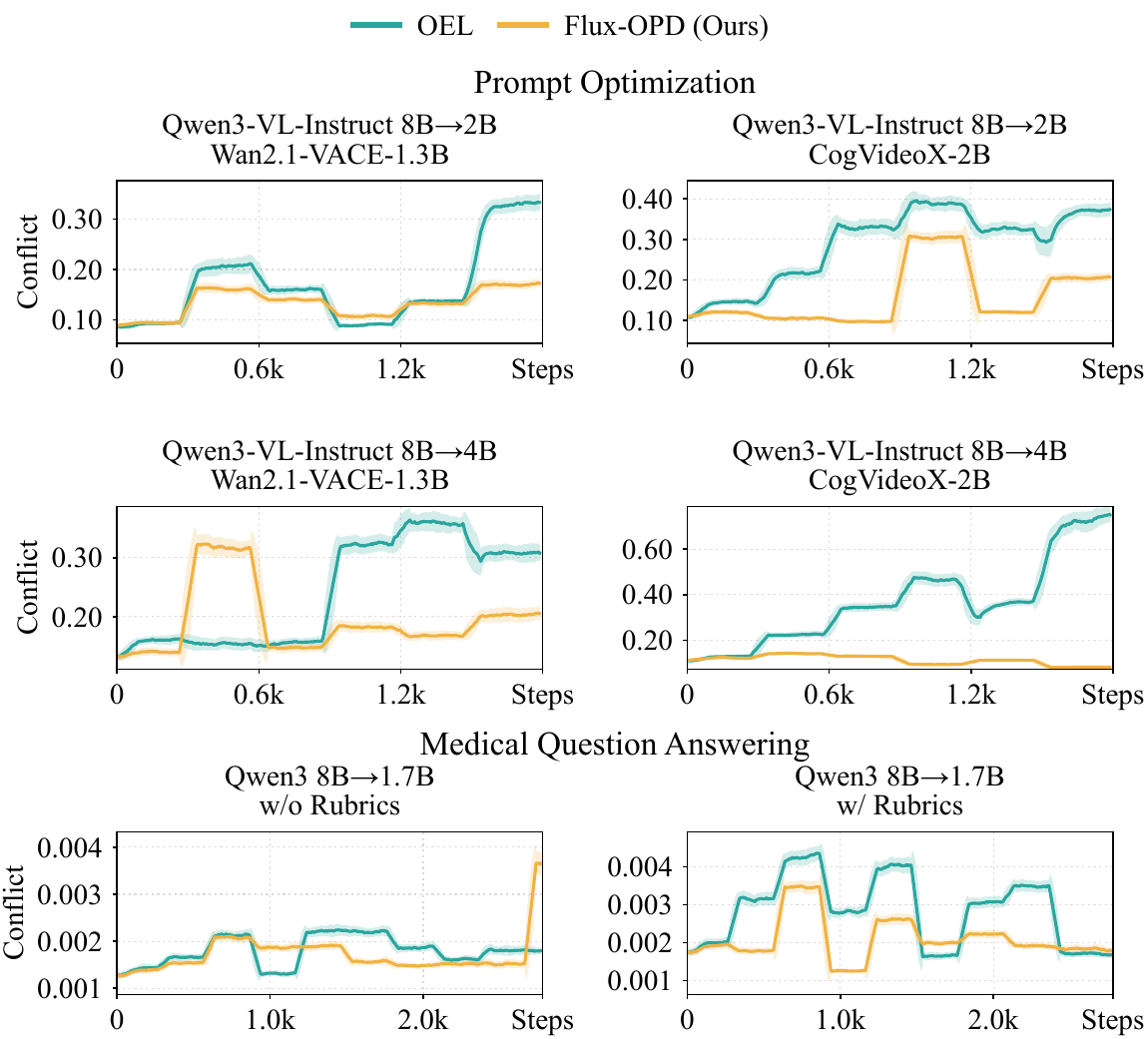}
    \caption{
    Conflict dynamics.
    The prompt optimization settings exhibit higher conflict values, whereas medical question answering maintains lower conflict, potentially because video sampling introduces greater variability and reduces context consistency.
    }

    \label{fig:supp_conflict}
\end{figure*}

\end{document}